\documentclass[letterpaper, 10 pt, conference]{ieeeconf}
\usepackage{amsmath,amsfonts,graphicx,array,hyperref}
\newtheorem{lemma}{Lemma}

\begin{document}

\title{Automatic generation of ground truth for the evaluation of obstacle detection and tracking techniques}

\author{
\authorblockN{Hatem Hajri\authorrefmark{1}, Emmanuel Doucet\authorrefmark{1}\authorrefmark{2}, Marc Revilloud\authorrefmark{1}, Lynda Halit\authorrefmark{1}, Benoit Lusetti\authorrefmark{1},\\  Mohamed-Cherif Rahal\authorrefmark{1}}\\
\authorblockA{\authorrefmark{1}\textit{Automated Driving Research Team}, Institut \textsc{VEDECOM}, Versailles, France}
\authorblockA{\authorrefmark{2}\textit{InnoCoRe Team}, Valeo, Bobigny, France}
}

\maketitle

\begin{abstract}
As automated vehicles are getting closer to becoming a reality, it will become mandatory to be able to characterise the performance of their obstacle detection systems. This validation process requires large amounts of ground-truth data, which is currently generated by manually annotation.
In this paper, we propose a novel methodology to generate ground-truth kinematics datasets for specific objects in real-world scenes. Our procedure requires no annotation whatsoever, human intervention being limited to sensors calibration. We present the recording platform which was exploited to acquire the reference data and a detailed and thorough analytical study of the propagation of errors in our procedure. This allows us to provide detailed precision metrics for each and every data item in our datasets. Finally some visualisations of the acquired data are given.

\end{abstract}


\section{Introduction}
Object detection and tracking both play a crucial role in autonomous driving. They are low-level functions upon which many other increasingly high-level functions are built. These functions include Intention prediction, Obstacle avoidance, Navigation and planning. Being depended on by so many functions, the task of obstacle detection and tracking must be performed with a high level of accuracy and be robust to varying environmental conditions. However, the generation of ground truth data to evaluate obstacle detection and tracking methods usually involves manual annotation, either of images, or of LIDAR point clouds \cite{Geiger2012CVPR}\cite{Cordts2016Cityscapes}\cite{8317828}. 

This paper showcases a method which takes advantage of the multiplication of autonomous driving platform prototypes in research structures to generate precise and accurate obstacle ground truth data, without requiring the usual phase of painstaking manual labelling of raw data.

Firstly, the methodology applied to generate this data will be described in general terms, and some specific technical topics such as the sensors time-synchronisation method used to collect data will be presented. Then analysis of errors propagation is performed. Finally some plots of the collected data are given together with potential applications. 

\section{General method description}
The proposed method requires two or more vehicles to generate obstacles dynamics ground truth data. The first vehicle -the ego-vehicle- is equipped with a high-precision positioning system (for example Global Navigation Satellite System (GNSS) with Real Time Kinematics (RTK) corrections coupled with an inertial measurement unit (IMU)), and various environment perception sensors (for example LIDARs, cameras or RADARs). The other vehicles -the target vehicles- only need to be equipped with a high-precision positioning system, similar to that of the ego-vehicle. By simultaneously recording the position and dynamics of all vehicles, it is possible to express the kinematics of all equipped vehicles present in the scene, in the ego-vehicle frame of reference, and therefore to generate reference data for these vehicles. This reference data can then be used to evaluate the performance of obstacle detection methods applied to the environmental sensors data collected on the ego-vehicle.

\section{Data collection setup}
In this section, we will present in details the set of sensors which were available for our data collection, and will detail the method applied to ensure the synchronicity of the recording process taking place in different vehicles. Three vehicles were used during this data collection campaign : the ego-vehicle was a Renault Scenic equipped to acquire environment perception data with a high accuracy. The target vehicles were Renault ZOEs, modified to be used as autonomous driving research platforms, and therefore equipped with precise positioning systems.

\subsection{Ego-vehicle perception sensors}
To record perception data, our ego-vehicle is equipped with two PointGrey 23S6C-C colour cameras, a Velodyne VPL-16 3D laser scanner (16 beams, 10Hz, 100m range, 0.2$^{\circ}$ horizontal resolution), a cocoon of five Ibeo LUX laser scanners (4 beams, 25Hz, 200m range, 0.25$^{\circ}$ horizontal resolution) covering a field of view of 360$^{\circ}$ around the vehicle.

\begin{figure}[ht]
    \centering
    \includegraphics[width = 0.8\linewidth]{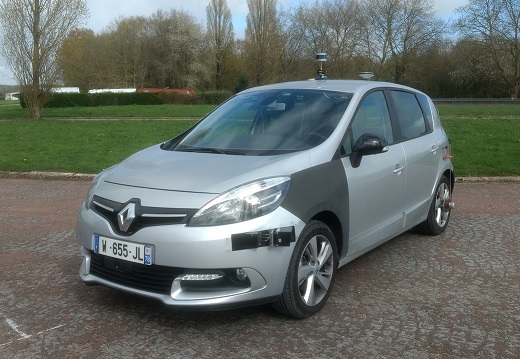}
    \centering
    \caption{The perception vehicle used : two ibeo LUX, the VLP16, GNSS antenna and the precision odometer are visible}
    \label{fig:scenic_img}
\end{figure}

The cameras are positioned inside the car, behind the windscreen with a baseline of approximately 50cm. The Velodyne VLP-16 is positioned on the roof of the vehicle at a position and height that minimise the occlusion of the laser beams by the vehicle body. The Ibeo LUX are all mounted at the same height of approximately 50cm, two on each front wing (one pointing forward, one to the side), and one at the back, pointing backwards (see Figure~\ref{fig:lux_setup} ).

\begin{figure}[ht]
    \centering
    \includegraphics[width = \linewidth]{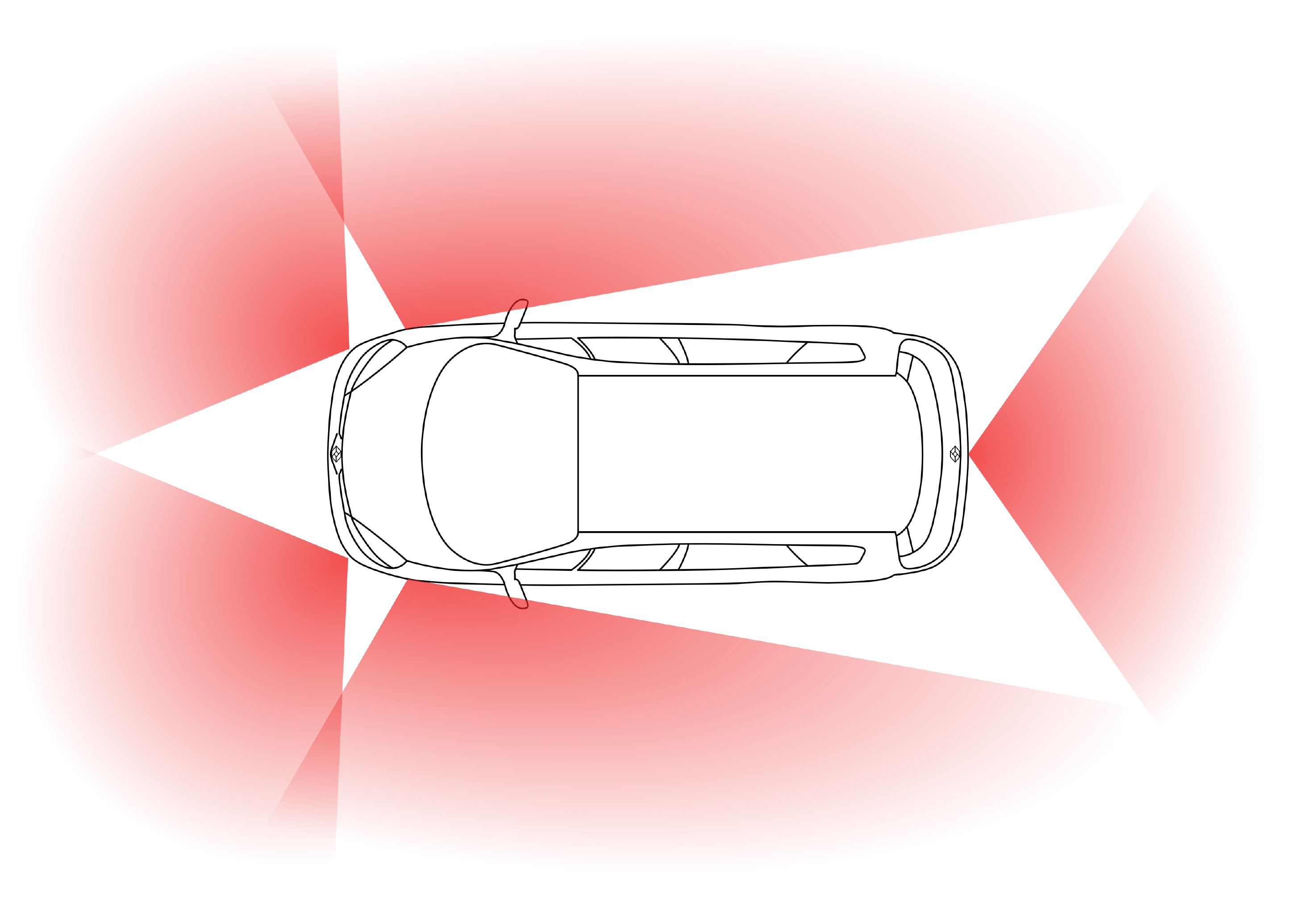}
    \caption{Ibeo LUX cocoon setup}
    \label{fig:lux_setup}
\end{figure}

\subsection{Precise localisation system}
The accuracy of the data generated using our method relies entirely on the accuracy of the vehicles' positioning systems. Therefore, each vehicle was equipped with state of the art positioning sensors : a choke-ring GNSS antenna feeding a GNSS-RTK receiver coupled with a high-grade, fibre optic gyroscopes-based iXblue Inertial Measurement Unit. Additionally, the perception vehicle is equipped with a high-accuracy odometer mounted on the rear-left wheel, while the target vehicles are equipped with a Correvit\textregistered\ high-accuracy, contact-less optical odometer. The data emanating from these sensors is fused using a Kalman Filter-based robust observer which jointly estimates the IMU and external sensors biases. The performance of this system can be further improved in post-processing by employing accurate ephemeris data and smoothing techniques. Table~\ref{tab:pos_perf} provides an overview of the combined performance of our positioning system and of the the aforementioned post-processing.

\begin{table}[ht]
    \renewcommand{\arraystretch}{1.3}
    \caption{Positioning system performance}
    \label{tab:pos_perf}
    \centering
    \begin{tabular}{|m{1.7cm}||>{\centering\arraybackslash}m{1.1cm}|>{\centering\arraybackslash}m{1.1cm}|>{\centering\arraybackslash}m{1.1cm}|>{\centering\arraybackslash}m{1.1cm}|}
        \hline
        & Heading (deg) & Roll/Pitch (deg) & Position X,Y (m) & Position Z (m)\\
         \hline\hline
        Nominal GNSS signal &  0.01 & 0.005 & 0.02 & 0.05\\\hline
        60 sec GNSS outage & 0.01 & 0.005 & 0.10 & 0.07\\\hline
        300 sec GNSS outage & 0.01 & 0.005 & 0.60 & 0.40\\\hline
    \end{tabular}
\end{table}

\subsection{Time synchronisation}
One of the biggest challenges of performing data collection distributed across multiple vehicles is to precisely synchronise the clocks used for time-stamping the data in each platform. This is especially true when acquiring data in high-speed scenarios. Highway scenarios, in which the absolute value of the relative velocity of vehicles may reach 70m/s, require the synchronisation of the vehicle clocks to be at least accurate to the millisecond. This inaccuracy induces an incompressible positioning error, which adds to that of our positioning system (see Subsection~\ref{subsec:sensitivity_pos}).

To achieve such a precise synchronisation, a Network Time Protocol (NTP) server fed with the pulse-per-second (PPS) signal provided by our GNSS receivers was installed in each vehicle to synchronise the on-board computers in charge of recording all the data. Prior to any recording session, the NTP servers and computer clocks were allowed 12 hours to converge to a common time.

\section{Sensors calibration}
Generating ground truth data requires very accurate sensors calibration. Given the difficulty of calibrating LIDAR sensors relatively to cameras\cite{Geiger2012CVPR}, we propose the following calibration pipeline : first, the positioning system is calibrated, then the cameras are calibrated intrinsically, and finally the rigid transformations relating cameras and LIDARs to the vehicle frame are estimated.

\subsection{Positioning system calibration}
The calibration of the positioning system consists in calculating the position and attitude of all positioning sensors (GNSS antenna, optical odometer) in the frame of reference of the IMU. After a phase of initialisation during which the vehicle remains static to allow the estimation of all sensors biases and the convergence of the RTK, the vehicle is manoeuvred. By comparing the motion information emanating from each sensor and comparing it to that of the IMU, one is then able to determine the rigid transformation between the said sensor and the IMU.

\subsection{Cameras calibration}
To estimate the intrinsic parameters of our cameras, and the relative pose of our stereo pair, we sweep the space in front of the cameras at three different depths, making sure to cover the whole field of view of each camera. We then use a mixture of Geiger's checkerboard pattern detector~\cite{Geiger2012ICRA} and of the sub-pixellic corner detector from openCV to extract the corners of the checkerboard. These are then fed to openCV's \texttt{stereoCalibrate} function to jointly estimate the intrinsic parameters of each camera and their relative pose. The set of parameters thus obtained typically yields re-projection errors of less than 0.3 pixel.

The position and orientation of our cameras are obtained in a semi-automatic fashion. Their position in the vehicle is precisely measured with laser tracers, and their orientations are estimated using a known ground pattern by minimising the re-projection error of the said pattern in the images.

\subsection{LIDARs calibration}
\subsubsection{Velodyne VLP-16 calibration}
The objective of the Velodyne to IMU calibration process is to determine the rigid transformation $T_{Vel\rightarrow IMU}$ between the Velodyne reference frame and that of the IMU.
Our Velodyne to IMU calibration process is fully automated. Using Iterative Closest Point, we match the 3D point clouds acquired during the calibration manoeuvre one to another to generate a trajectory. Then, the pose of the IMU is re-sampled to the timestamps of the Velodyne scans, and $T_{Vel\rightarrow IMU}$ is estimated through a non-linear optimisation process applied to 1000 pose samples. This rigid transformation estimate is then refined by repeating the process, using $T_{Vel\rightarrow IMU}$ and the linear and angular velocities of the vehicle to correct the motion-induced distorsion of the Velodyne point clouds. This process usually converges in just one iteration.
\subsubsection{Ibeo LUX}
The calibration of an Ibeo LUX cocoon is slightly more complicated than that of a single Velodyne, as it involves simultaneously calibrating all the sensors. Indeed, calibrating each LUX separately will almost certainly result in a poorly consistent point cloud when aggregating clouds from all sensors. Likewise, precisely calibrating one sensor, and then calibrating all sensors relative to their neighbour will also lead to such poor results. A simple way to ensure the global coherence of the cocoon calibration is to use the point cloud from a calibrated Velodyne as a reference, and to calibrate all Ibeo LUX sensors relative to this spatially coherent reference.

\section{Ground-truth data generation}\label{sec:process}
\subsection{Notations}
Let us call $X_i^k = (x,y,v_x,v_y,\psi)_i^k$ the state of vehicle $i$, with $(x,y)_i^k$ the position of its reference point, $(v_x,v_y)_i^k$ its velocity vector, and $\psi_i^k$ its yaw angle, all expressed in the frame of reference $k$.

In the rest of the paper, $\cdot_e$ and $\cdot_t$ respectively denote state variables of the ego and of the target vehicle, and $\cdot^{UTM}$ and $\cdot^{ego}$ respectively denote a variable expressed in the Universal Transverse Mercator and in the ego-vehicle frame of reference.

\subsection{Processing}
Generating a set of obstacle ground truth data from high accuracy positioning recordings is a two step process :
\begin{itemize}
    \item generate the relative position and dynamics of the obstacles relative to the ego-vehicle,
    \item generate data carrying obstacle semantics from the previously generated data.
\end{itemize}

For each sensor recording, a whole set of ground truth data is generated, so as to provide ground truth synchronised with the sensor data. At each sensor data timestamp, the position and dynamics of each vehicle are estimated from the positioning system recording using temporal splines interpolation.

From this, simple kinematics and velocity composition formulae allow the computation of the relative position and dynamics of the target vehicles in the ego-vehicle frame of reference :

\begin{equation}\label{eq:pos}
     \left[\begin{array}{c} x \\ y \end{array}\right]_t^{ego} = R(-\psi_e^{UTM}) \left[\begin{array}{c} x_t - x_e \\ y_t-y_e \end{array}\right]^{UTM}
\end{equation}

\begin{equation}\label{eq:vel}
    \left[\begin{array}{c} v_x \\ v_y \end{array}\right]_t^{ego} = R(-\psi_e^{UTM}) \left[\begin{array}{c} {v_x}_t - {v_x}_e + \dot{\psi_e}(y_t - y_e) \\ {v_y}_t - {v_y}_e - \dot{\psi_e}(x_t - x_e) \end{array}\right]^{UTM}
\end{equation}

\begin{equation}\label{eq:yaw}
    \psi_t^{ego} = \psi_t^{UTM} - \psi_e^{UTM}
\end{equation}

Where $R(\alpha)$ denotes a rotation in $\text{SO}_2$ of angle $\alpha$

\subsection{Exploitation}
The reference relative positioning data thus obtained can then be used to generate ground truth perception data. One can for example generate the bounding box of the target vehicle in the ego-vehicle frame of reference to evaluate the performance of LIDAR or image-based object detection and tracking algorithms. Another possibility is to use 3D models of the target vehicles and to project them in the camera images to automatically generate a partial image segmentation.

\section{Uncertainty propagation analysis} \label{sec:sensitivity}
In this section, we perform a sensitivity analysis of our ground truth generation process, to characterise the accuracy and precision of the generated data, depending on the performance of our positioning systems, and the clock shift between the vehicles.

The inputs of our generation process are made of position, velocity and heading estimates provided by a GNSS-INS fusion system. These can be modelled as independent, Gaussian random variables\cite{Niu2014}. 

Therefore, we will treat the position, velocity and yaw angle separately, so as to limit the calculation hurdle.

\subsection{Position}\label{subsec:sensitivity_pos}
Equation (\ref{eq:pos}) yields :
$$\left[\begin{array}{c} x \\ y \end{array}\right]_t^{ego} = F(dx^{UTM}, dy^{UTM}, \psi_e^{UTM}) $$
with :
$$F(dx,dy,\psi_e) = \left[\begin{array}{c} dx \cos{\psi_e} + dy \sin{\psi_e} \\ dy\cos{\psi_e} - dx\sin{\psi_e}\end{array}\right].$$
\begin{lemma}\label{techn}
	Let $\Omega$ be a Gaussian random variable such that $\Omega \sim \mathcal{N}(m_\Omega,\sigma_\Omega^2)$.
	Then:
	\begin{align*}
	    \mathbb{E}(\cos(\Omega)) &=\cos(m_\Omega)\ e^{-\sigma_\Omega^2/2}\\
	    \mathbb{E}(\sin(\Omega)) &=\sin(m_\Omega)\ e^{-\sigma_\Omega^2/2}
	\end{align*}
\end{lemma}
\begin{proof}
	Using the explicit expression of the characteristic function of a Gaussian variable, we get:
	\begin{equation*}
	    \mathbb{E}(\cos(\Omega)+i\sin(\Omega))=\mathbb{E}(e^{i\Omega})=e^{i\,m_\Omega-\sigma_\Omega^2/2}
	\end{equation*}
	The real and imaginary part of this expression yield the desired result.
\end{proof}
\vspace{1em}
Under the assumption that $\mathrm{Var}(dx)=\mathrm{Var}(dy)=\sigma_{dx}^2$, and noting $\mathrm{Var}(\psi_e) = \sigma_\psi^2$ : 
$$\mathrm{Cov}(F(dx,dy,\psi_e)) = \left[\begin{matrix}a & c \\ c & b \end{matrix}\right]$$
With :
\begin{align*}
a &= \sigma_{dx}^2+\mathbb{E}(dx)^2\mathrm{Var}(\cos{\psi_e})+\mathbb{E}(dy)^2\mathrm{Var}(\sin{\psi_e})\\
  &- \mathbb{E}(dx)\mathbb{E}(dy)\sin{(2\mathbb{E}(\psi_e))}e^{-\sigma_\psi^2}(1-e^{-\sigma_\psi^2})\\[1em]
b &= \sigma_{dx}^2+\mathbb{E}(dx)^2\mathrm{Var}(\sin{\psi_e})+\mathbb{E}(dy)^2\mathrm{Var}(\cos{\psi_e})\\
  &+ \mathbb{E}(dx)\mathbb{E}(dy)\sin{(2\mathbb{E}(\psi_e))}e^{-\sigma_\psi^2}(1-e^{-\sigma_\psi^2})\\[1em]
c &= \frac{1}{2}\sin{(2\mathbb{E}(\psi_e))}e^{-\sigma_\psi^2}(1-e^{-\sigma_\psi^2})(\mathbb{E}(dx)^2-\mathbb{E}(dy)^2)\\
  &- \mathbb{E}(dx)\mathbb{E}(dy)\cos{(2\mathbb{E}(\psi_e))}e^{-\sigma_\psi^2}(1-e^{-\sigma_\psi^2})
\end{align*}

Therefore, under the assumption that the variance of the position error is similar from one vehicle to another, and along North and East axes ($\sigma_{x}^2 = \sigma_{y}^2 = \sigma_{pos}^2=\frac{1}{2} \sigma_{dx}^2)$), and that the maximal distance between the ego-vehicle and an obstacle is $d_{max}$, we propose the following upper bound for the position error covariance matrix:
\begin{align*}
a &\leq 2\sigma_{pos}^2 + 2 d_{max}^2(1-e^{-\sigma_\psi^2}), \\
b &\leq 2\sigma_{pos}^2 + 2 d_{max}^2(1-e^{-\sigma_\psi^2}),\\
c &\leq \frac{3}{2} d_{max}^2(1-e^{-\sigma_\psi^2/2}).
\end{align*}

\subsection{Velocity}

Equation (\ref{eq:vel}) yields :
$$\left[\begin{array}{c} v_x \\ v_y \end{array}\right]_t^{ego} = G\left((dx, dy,dv_x, dv_y, \psi_e, \dot{\psi_e})^{UTM}\right) $$
with $G(dx, dy,dv_x, dv_y, \psi_e, \dot{\psi_e})$ equal to :

\begin{equation*}
\left[\begin{matrix}
\cos{\psi_e}(dv_x+\dot{\psi_e}dy) + \sin{\psi_e}(dv_y-\dot{\psi_e}dx) \\
\cos{\psi_e}(dv_y-\dot{\psi_e}dx) - \sin{\psi_e}(dv_x+\dot{\psi_e}dy) 
\end{matrix}\right]
\end{equation*}

Using the independence of the input variables, it is possible to express the covariance matrix of $G(dx, dy,dv_x, dv_y, \psi_e, \dot{\psi_e})$ as a function of the first and second order moments of the input variables. Due to space limitations, we only give bounds for the elements of this matrix. Similarly to Section (\ref{subsec:sensitivity_pos}), these bounds tend to $0$ as the variances of the input variables tend to $0$.

\begin{lemma}
	Let $X$ and $Y$ be two independent random variables with means $m_X$ and $m_Y$ and variances $\sigma^2_X$ and $\sigma^2_Y$. Let $\Omega$ be a Gaussian random variable $\mathcal{N}(m_\Omega,\sigma_\Omega^2)$ independent of $X, Y$. Let $Z=\cos(\Omega) X + \sin(\Omega) Y$. Then :
	\begin{equation*}
	    \text{Var}\ (Z)\le \sigma^2_X + \sigma^2_Y + (1- e^{-\sigma_{\Omega}^2}) (|m_X|+|m_Y|)^2
	\end{equation*}
\end{lemma}
\vspace{1em}
\begin{proof}
Using Lemma \ref{techn}, we have the following bound:
\begin{align*}
\mathrm{Var}(Z) &\leq \sigma^2_X + \sigma^2_Y + m_X^2 \mathbb{E} (\cos^2(\Omega)) +  m_Y^2 \mathbb{E} (\sin^2(\Omega))\nonumber\\
&- e^{-\sigma_{\Omega}^2}(m_X^2 \cos^2(m_\Omega) +  m_Y^2 \sin^2(m_\Omega))\\
&- m_X m_Y \sin(2 m_{\Omega}) e^{-\sigma_{\Omega}^2}(1-e^{-\sigma_{\Omega}^2})
\end{align*}
Using the identity $\cos^2(\Omega)=\frac{1+2 \cos(\Omega)}{2}$ and again Lemma~\ref{techn}, we get	:
\begin{equation*}
    m_X^2 \mathbb{E} (\cos^2(\Omega))  - e^{-\sigma_{\Omega}^2} m_X^2 \cos^2(m_\Omega) \le m_X^2 (1- e^{-\sigma_{\Omega}^2})
\end{equation*}

The same identity holds with $X$ replaced with $Y$ and $\cos$ replaced with $\sin$. This gives the desired bound.
\end{proof}
\vspace{1em}
Let us denote the mean and variance of a Gaussian variable $Z$ respectively by $m_Z$ and $\sigma^2_Z$. Under the assumption that $\mathrm{Var}(dv_x) = \mathrm{Var}(dv_y) =2\sigma_{vel}^2$ and $\mathrm{Var}(dx)=\mathrm{Var}(dy)=\sigma_{dx}^2$, we have :
\begin{align*}
    \mathrm{Var}(dv_x+\dot{\psi}dy) &= 2\sigma_{vel}^2 + \sigma_{dx}^2\sigma_{\dot{\psi}}^2+m_{\dot{\psi}}^2 \sigma_{dx}^2+m_{dy}^2\sigma_{\dot{\psi}}^2 \\[0.5em]
    \mathrm{Var}(dv_y-\dot{\psi}dx) &=\mathrm{Var}(dv_x+\dot{\psi}dy) 
\end{align*}

Let $\left[\begin{matrix}a & c \\ c & b \end{matrix}\right]$ be the covariance matrix of $\left[\begin{array}{c} v_x \\ v_y \end{array}\right]_t^{ego}$. Using the previous results, we have the following upper bounds for $a, b$ and $c$ :
\begin{align*}
a,b &\leq 4(\sigma_{vel}^2 + \sigma_{pos}^2\sigma_{\dot{\psi}}^2 + \dot{\psi}_{max}^2 \sigma_{pos}^2) + 2d_{max}^2\sigma_{\dot{\psi}}^2,\nonumber\\
&+4(1- e^{-\sigma_{\psi}^2}) (v_{max}+d_{max}\dot{\psi}_{max})^2\\[1em]
c &\leq \sqrt{a^2} \sqrt{b^2},
\end{align*}
where $v_{max}$ is an upper bound for $dv_x$ and $dv_y$ and $\dot{\psi}_{max}$ is an upper bound for $\dot{\psi}$.
\subsection{Yaw angle}

The variance of the relative yaw estimate $\psi_t^{ego}$ is trivially derived from equation (\ref{eq:yaw}). Considering that the heading error covariance is similar in the target and ego-vehicle ($\sigma_{\psi_t^{UTM}}^2 = \sigma_{\psi_e^{UTM}}^2 = \sigma_\psi^2$), we get:
\begin{equation*}
\mathrm{Var}(\psi_t^{ego}) = 2\sigma_\psi^2
\end{equation*}

\subsection{Results}
Using the analytical upper bounds for the covariance matrices of the generated ground truth position, speed and yaw angle data calculated in the previous sections, we can estimate the precision of the data generated using the process described in section~\ref{sec:process}.

The following values will be used : 
\begin{table}[h]
    \renewcommand{\arraystretch}{1.3}
    \centering
    \begin{tabular}{|l|r l|}
        \hline
        $\sigma_{pos}$ & 0.02 & m\\\hline
        $\sigma_{vel}$ & 0.02 & m.s$^{-1}$\\\hline
        $\sigma_\psi$ & $1.75 . 10^{-3}$ & rad\\\hline
        $d_{max}$ & 50 & m\\\hline
        $v_{max}$ & 36 & m.s$^{-1}$\\\hline
        $\dot{\psi}_{max}$ & 1 & rad.s$^{-1}$\\\hline
    \end{tabular}
    \label{tab:est_values}
\end{table}

They are representative of the performance of our positioning system, of the distance at which obstacles become truly observable, and of the maximal dynamics of the vehicle in typical use cases.

These values, when injected in the upper bounds of the position and velocity covariance matrices yield the following results :
\begin{align*}
    \|\mathrm{Cov}((x,y)_t^{ego})\|_F^{1/2} &\leq 0.12\ m\\
    \|\mathrm{Cov}((v_x,v_y)_t^{ego})\|_F^{1/2} &\leq 0.30\ m.s^{-1}
\end{align*}

These values represent the root mean square error on the position and velocity information yielded by our ground truth generation process. We stress the fact that they are obtained by taking extreme values for all variables describing the dynamics of both vehicles, which can never actually be encountered simultaneously in real life situations. 
\section{Data visualisation}
In this section some plots of acquired data are given. These data were collected during a tracking scenario on a track located in Versailles France. The track has approximative length of 3.2 km and is represented in Figure 3 together with the starting and arriving points of the vehicles.
\begin{figure}[h!]
	\begin{center}
		\includegraphics[width=8.1cm]{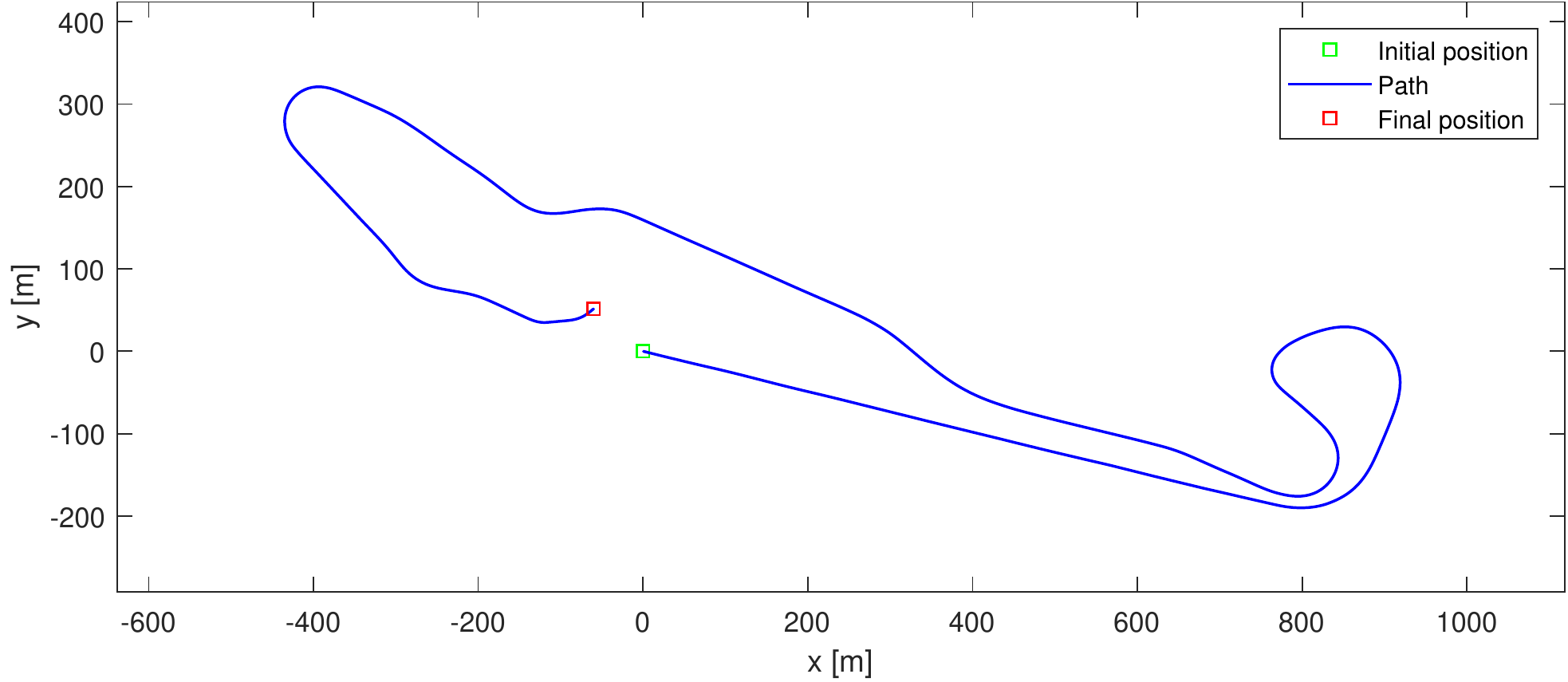} 
		\caption{Test track}
		\label{Satory test track}
	\end{center}
\end{figure}

The following figures show temporal variations of the relative positions $x, y$, relative velocities $vx, vy$ and orientation $\psi$ sent by Lidar of the tracked vehicle (in the ego vehicle frame) and the corresponding ground truth plots. All these quantities are in the  International System of Units. 

\begin{figure}[h!]
	\begin{center}
		\includegraphics[width=8.1cm]{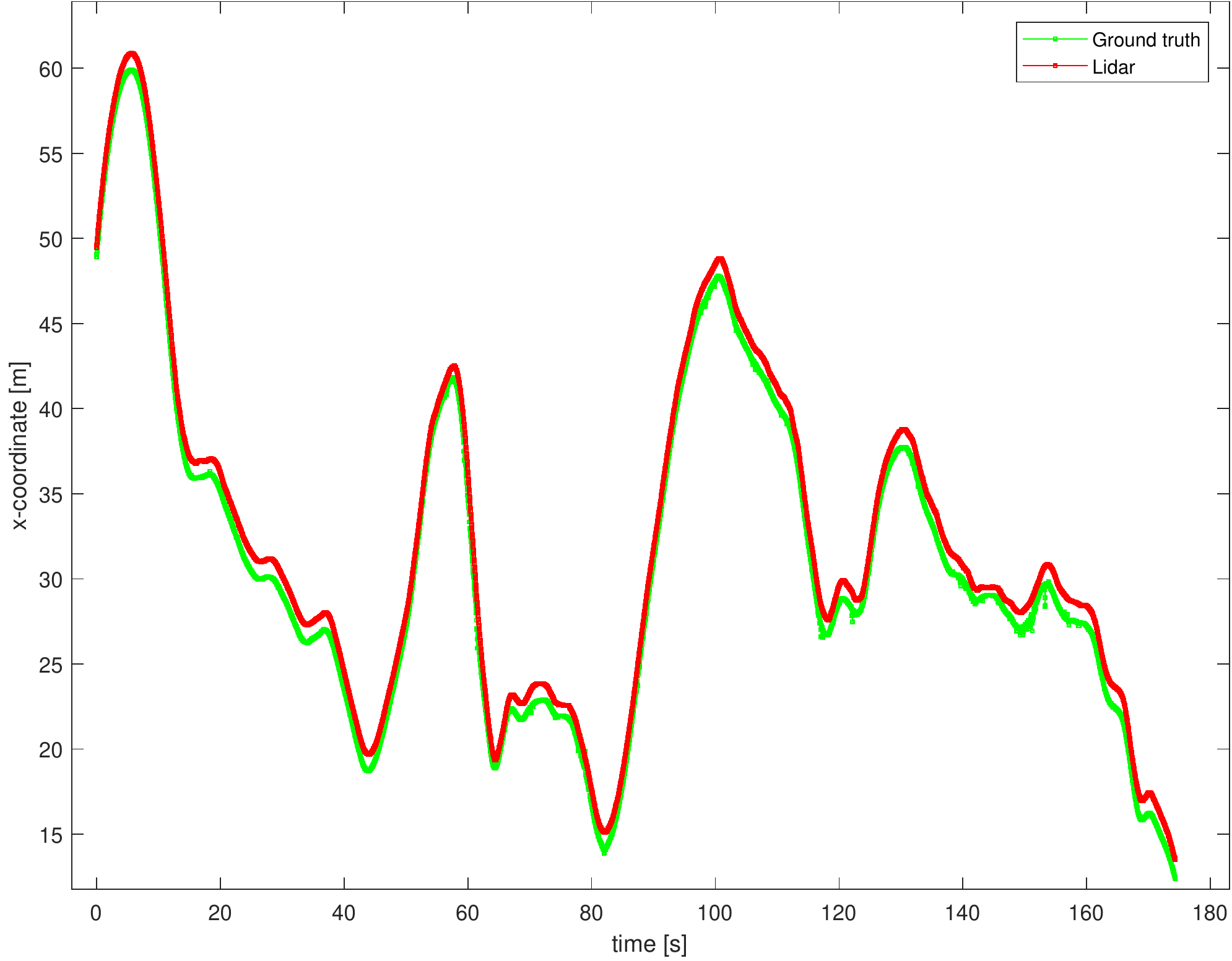} 
		\caption{Variations of $x$}
		\label{Satory test track}
	\end{center}
\end{figure}

\begin{figure}[h!]
	\begin{center}
		\includegraphics[width=8.1cm]{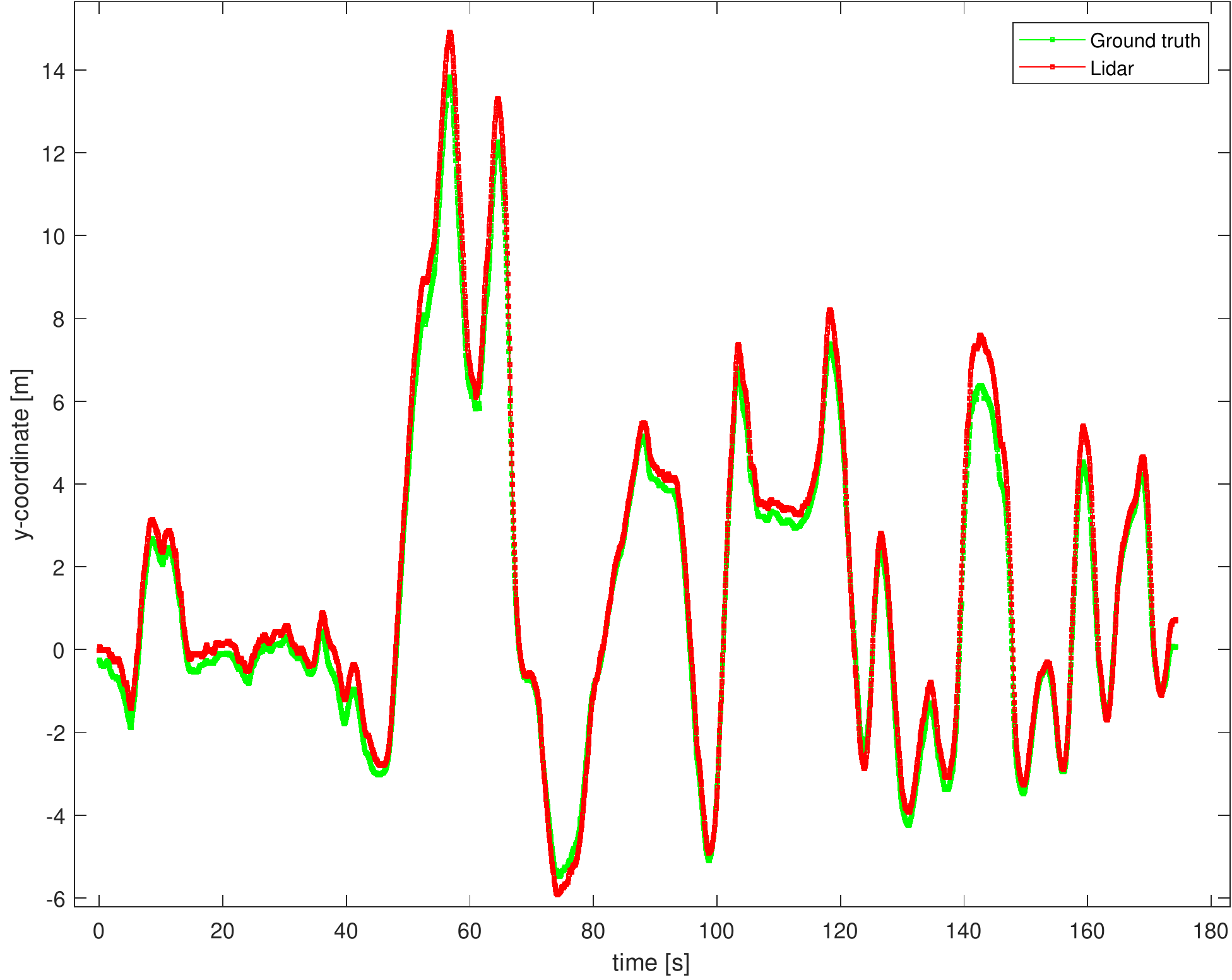} 
		\caption{Variations of $y$}
		\label{Satory test track}
	\end{center}
\end{figure}

\begin{figure}[h!]
	\begin{center}
		\includegraphics[width=8.1cm]{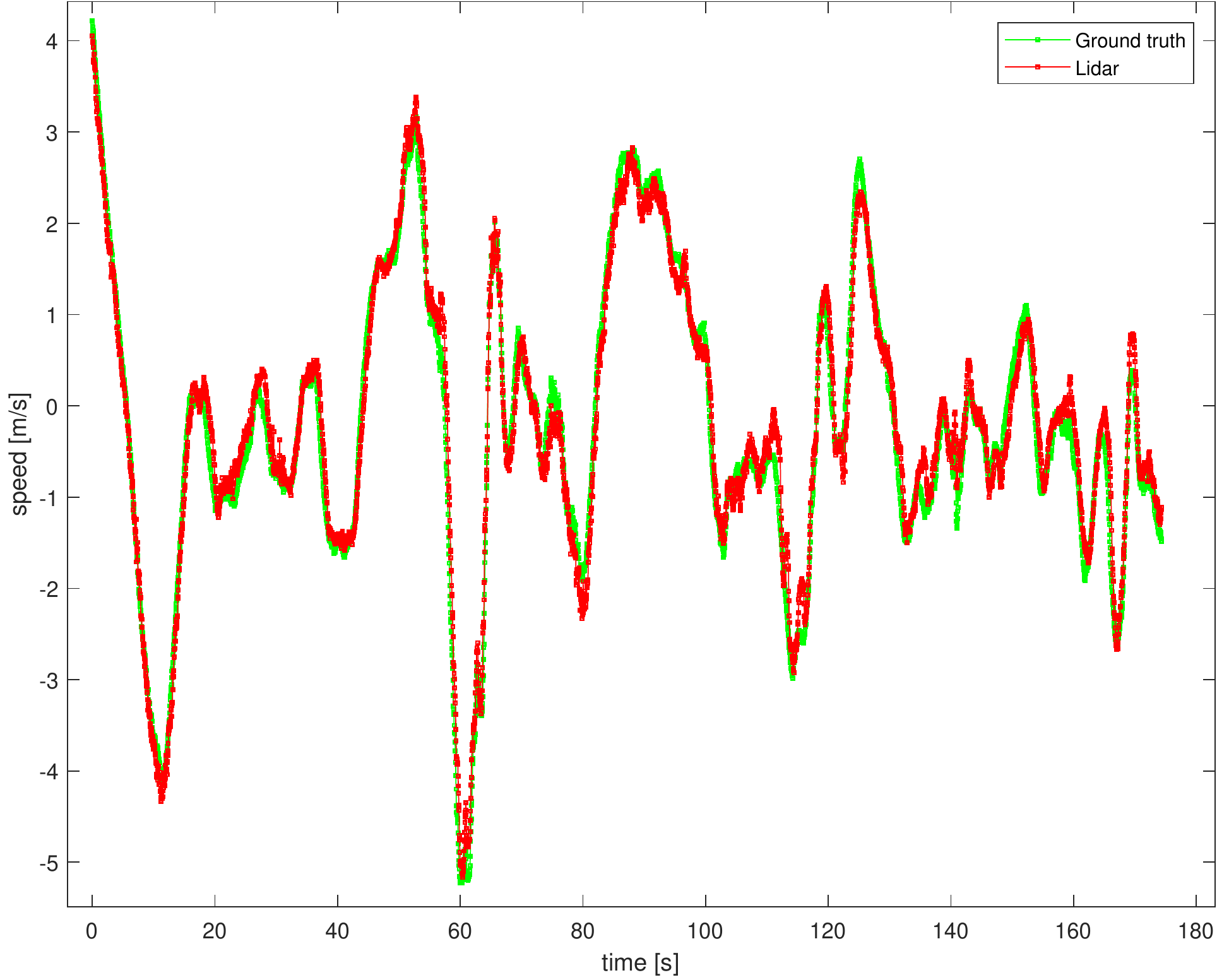} 
		\caption{Variations of $vx$}
		\label{Satory test track}
	\end{center}
\end{figure}

\begin{figure}[h!]
	\begin{center}
		\includegraphics[width=8.1cm]{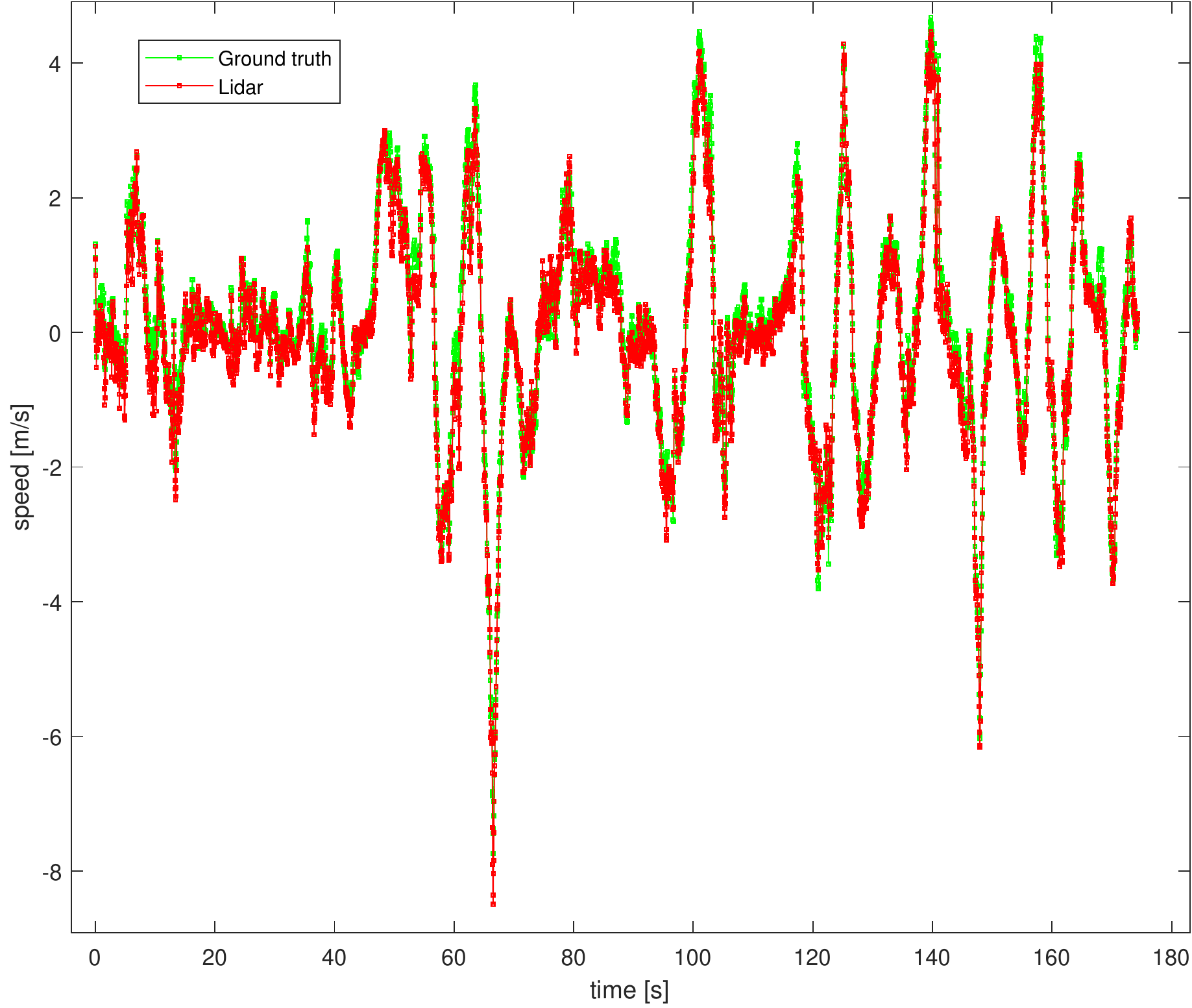} 
		\caption{Variations of $vy$}
		\label{Satory test track}
	\end{center}
\end{figure}

\begin{figure}[h!]
	\begin{center}
		\includegraphics[width=8.1cm]{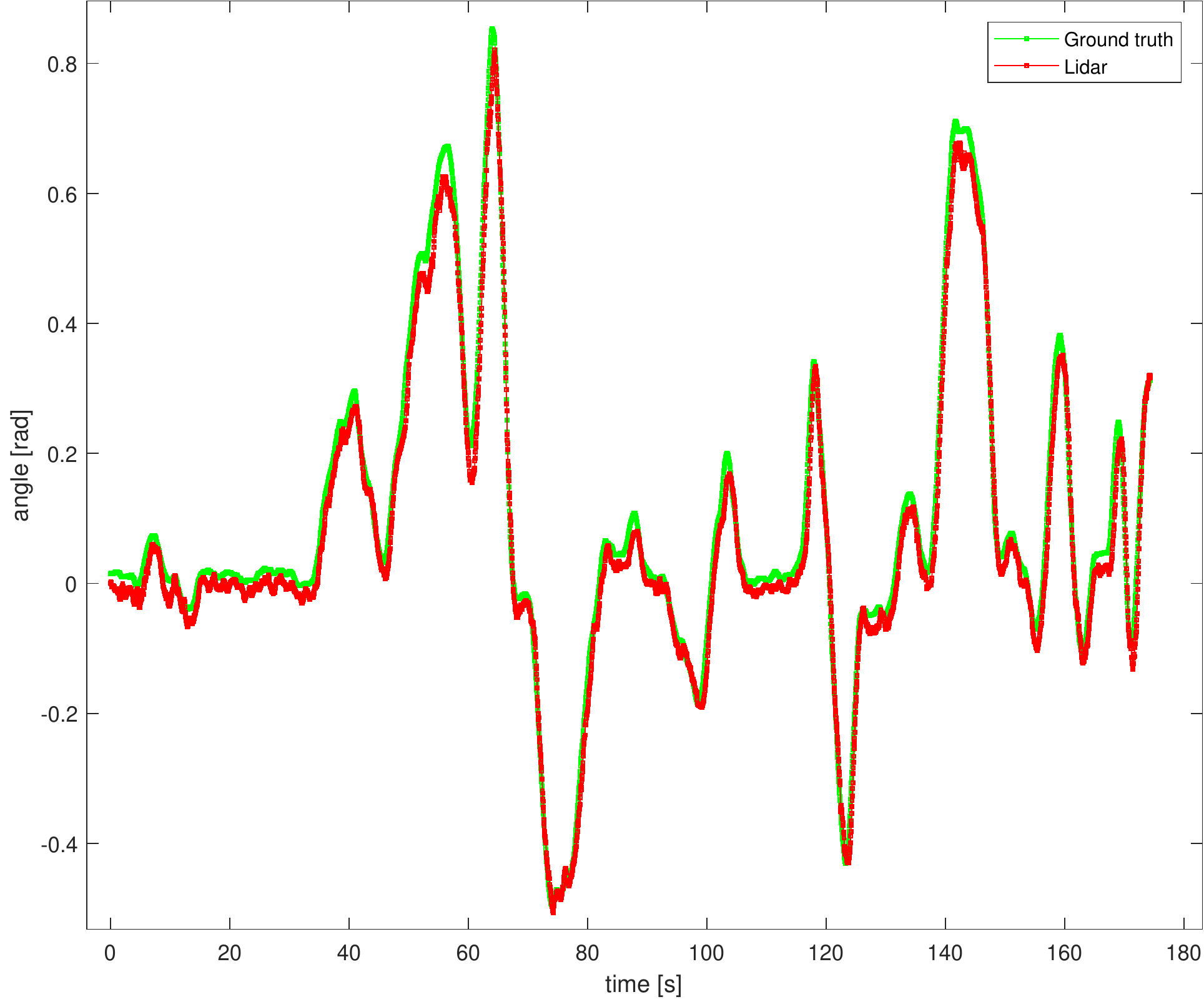} 
		\caption{Variations of $\psi$}
		\label{Satory test track}
	\end{center}
\end{figure}
\newpage
\section{Application to the evaluation of perception algorithms}
This section gives some potential applications of ground truth data to the conception and evaluation of perception algorithms.
\subsection{Clustering Lidar raw data:} Robust obstacle detection using sensors such as Lidar is a key point for the development of autonomous vehicles. Lidar performance mainly depends on its low level processing (the clustering methods used for its raw data, estimations of the bounding boxes, velocities of the representative points of the clusters, etc.). The methodology presented in the paper can be used to evaluate the performance of the main outputs of a Lidar raw data clustering algorithm.  
 
\subsection{Lidar-Radar-Camera Fusion:} Lidars, Radars and Cameras are complementary sensors. Lidars are very accurate on obstacles positions and less accurate on their velocities. On the other hand, Radars are more precise on obstacles velocities and less precise on their positions. Cameras provide images and can be used to perform classification tasks. Fusion between these sensors aims at combining the advantages of each sensor to provide permanent and more robust data (see for example \cite{HLBM,Blanctrackto,ttt,DBLP:journals/tits/GarciaA16,AM}). In particular, the recent work \cite{HLBM} proposes a Lidar-Radar fusion algorithm with evaluation on the database generated in the present paper.
  
\subsection{Dynamic models and prediction of obstacles motions: } Robust tracking of obstacles detected by sensors such as Lidars, Radars and Cameras is crucial for a good functioning of autonomous cars. 
Kalman filters are among the most used methods to track obstacles. These methods are usually coupled to preset models such as constant velocity, acceleration or curvature. Ground truth allows to learn true dynamic models of obstacles using statistics, neural networks etc. The learnt models are to be compared with the preset ones.

\section{Summary and future developments}
In this paper, we have presented a method for automatically generating sets of ground truth data to support advances in the field of obstacle detection and tracking. We hope this methodology will help contributors of this area of research to challenge their approaches, and contribute to the development of robust and reliable algorithms. 
In the future, we intend to propose a complete benchmark to unify the usage of this dataset and the performance estimation of obstacle detection and tracking techniques. In addition, we plan to apply statistical and deep learning approaches to the generated ground truth in order to correct sensor measurements, learn motion models etc.

\bibliographystyle{IEEEtran}
\bibliography{IEEEabrv,biblio}

\end{document}